\documentclass[letterpaper, 10 pt, conference]{ieeeconf}  

\IEEEoverridecommandlockouts                              

\usepackage{amsmath} 
\usepackage{amssymb}  

\usepackage{amsthm}
\usepackage[english]{babel}
\usepackage[utf8]{inputenc}
\usepackage{graphics} 
\usepackage{epsfig}
\usepackage{color}
\usepackage[dvipsnames]{xcolor}
\usepackage[hidelinks]{hyperref}
\usepackage{subfigure}
\usepackage{cite}
\usepackage{siunitx}

\newcommand{\human}{\xi}

\newtheorem{proposition}{Proposition}

\title{\LARGE \bf
HMPCC: Human-Aware Model Predictive Coverage Control
}

\author{Mattia Catellani$^*$,  Marta Gabbi$^*$, Lorenzo Sabattini
\thanks{*These authors contributed equally to this work.}
\thanks{The authors are with Department of Sciences and Methods for Engineering, University of Modena and Reggio Emilia, Italy 
{\tt\small \{mattia.catellani, marta.gabbi, lorenzo.sabattini\}@unimore.it}%
}%
\thanks{This work was supported by the AI-DROW Project, funded by the Italian Ministry of University and Research under the PRIN 2022 program.}
}

\begin{document}

\maketitle
\thispagestyle{empty}
\pagestyle{empty}

\begin{abstract}

We address the problem of coordinating a team of robots to cover an unknown environment while ensuring safe operation and avoiding collisions with non-cooperative agents. Traditional coverage strategies often rely on simplified assumptions, such as known or convex environments and static density functions, and struggle to adapt to real-world scenarios, especially when humans are involved.
In this work, we propose a human-aware coverage framework based on Model Predictive Control (MPC), namely HMPCC, where human motion predictions are integrated into the planning process. By anticipating human trajectories within the MPC horizon, robots can proactively coordinate their actions 
and adapt to dynamic conditions.
The environment is modeled as a Gaussian Mixture Model (GMM), representing regions of interest. Team members operate in a fully decentralized manner, without relying on explicit communication—an essential feature in hostile or communication-limited scenarios.
Our results show that human trajectory forecasting enables more efficient and adaptive coverage, improving coordination between human and robotic agents.

\end{abstract}

\section{INTRODUCTION}



Tasks such as search and rescue, environmental monitoring, exploration, surveillance, and inspection often occur in hostile and critical scenarios, where teams of robots are deployed in spaces potentially shared with humans. In such scenarios, humans behavior can be unpredictable in both their decision-making and movements, increasing the uncertainty and complexity of the environment. 
Therefore, it is essential to coordinate the motion of robots to account for the presence and behavior of nearby humans while exploring an unknown environment.


A common solution to this coordination problem is area coverage. 
Traditional methods like Voronoi tessellation~\cite{cortes2004coverage} and Lloyd's algorithm~\cite{lloyd1982least} partition the space and position robots at the centroids of their cells, often guided by a density function representing areas of interest. However, these approaches typically assume convex environments, known density distributions, and simple robot dynamics.

One of our key contributions is to overcome these limitations and extend coverage techniques to more realistic and complex scenarios. The challenge intensifies when robots must operate alongside humans. 
While prior work often treats human presence reactively or with static models, we integrate human trajectory prediction (HTP) to proactively and safely coordinate robot motion, ensuring adaptive and reliable area coverage. 
HTP has been studied for decades and is increasingly relevant in domains such as autonomous driving, social robotics, and surveillance.
According to~\cite{rudenko}, HTP methods are broadly classified into modeling-based and context-based approaches. Modeling-based methods use physical principles and domain knowledge, including crowd models like the Social Force Model (SFM~\cite{SFM}) and velocity-based techniques such as Reciprocal Velocity Obstacles (RVO~\cite{RVO}) and Optimal Reciprocal Collision Avoidance (ORCA~\cite{ORCA}). Context-based methods rely on internal and external stimuli, often using data-driven deep learning models like Long Short-Term Memory networks (e.g., SocialLSTM~\cite{alahi2016social}), Generative Adversarial Networks (e.g., SocialGAN~\cite{gupta2018social}), and Conditional Variational Autoencoders (e.g., Trajectron++~\cite{salzmann2020trajectron++}). FlowChain~\cite{maeda2023fast}, a flow-based model using continuously-indexed flows, is another notable example.
\begin{figure}[t]
    \centering
    \includegraphics[width=0.95\linewidth]{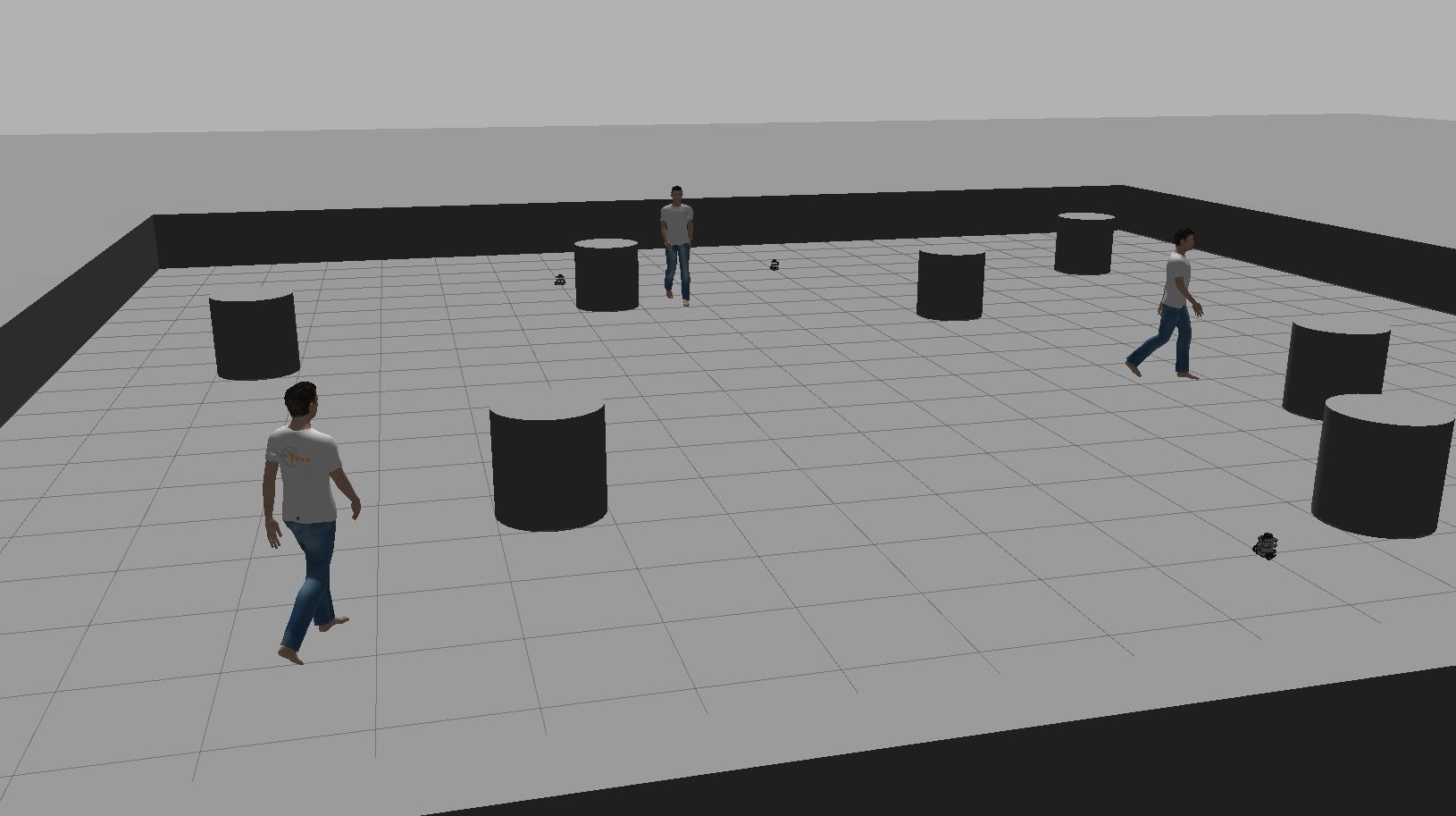}
    \caption{Gazebo simulation with humans and TurtleBot3 robots. HMPCC integrates human motion prediction into trajectory planning, successfully guiding robots with arbitrary dynamics in obstacle-rich environments.}
    \label{fig:gazebo}
\end{figure}
In this work, we address optimal coverage of an unknown environment by a team of robots that share the environment with humans. 
Exploration is guided by a Gaussian-mixture density map, while all agents operate in a decentralized manner and without communication, reflecting communication-denied or hostile settings. Robot motion is planned with an MPC framework that embeds human trajectory prediction, namely HMPCC. Simulations as in Fig.~\ref{fig:gazebo} prove that this method ensures safe and effective coverage even when human intentions are unknown. 



\subsection*{Contributions}
The main contributions of this work are twofold and can be summarized as follows:
\begin{itemize}
    \item an MPC-based framework for coverage control that accounts for generic robot's dynamics and non-convexity of the environment; and 
    \item the integration of human trajectory prediction into the MPC formulation to account for human motion when controlling the robots.
\end{itemize}

\section{RELATED WORKS}\label{sec:rel_works}
Several works have extended coverage control to more realistic scenarios, such as non-convex environments and general robot dynamics. In~\cite{breitenmoser2010voronoi}, a decentralized Voronoi coverage approach is proposed for polygonal environments using geodesic-based tessellation and the TangentBug~\cite{kamon1998tangentbug} algorithm for local planning. Battacharya et al.~\cite{bhattacharya2013distributed} introduce an entropy-weighted Voronoi partition, guiding robots toward high-entropy regions. Potential-based repulsive forces have also been used for collision avoidance~\cite{franco2015persistent}. Instead, dynamics non-linearity has been recently addressed in~\cite{liu2024distributed} for unicycle robots. MPC has also emerged as an important technique to handle constraints and non-linear dynamics. In~\cite{carron2020model}, an MPC-based coverage algorithm achieves convergence to a centroidal Voronoi configuration, assuming known density functions. Rickenbach et al.~\cite{rickenbach2024active} further extend MPC to unknown environments by incorporating active density learning.

Despite these improvements over traditional methods, explicitly accounting for humans in the coverage problem remains a key challenge that is still largely unexplored. Unlike robots, whose behavior can be governed by control policies, humans are autonomous agents who make independent decisions. This makes their behavior significantly more difficult to model, highlighting the need to consider human presence as part of the coverage problem. 

In the literature, many works have investigated how to account for human presence in coverage and navigation tasks.
In \cite{breitenmoser2016combining}, the authors studied different collision scenarios, including a heterogeneous crowd of both robots and humans, who are modeled with a holonomic kinematic model and are distributed according to the Voronoi coverage controller. 
Claes et al. \cite{Claes2018} addressed the problem of collision avoidance in a shared space by employing various cost maps and Monte Carlo sampling, incorporating different cost factors to account for both humans and robots.
The task of multi-robot coordination in environments shared with humans is tackled by Talebpour et al. \cite{Talebpour}. Humans are modeled as dynamic obstacles with orientation-based Gaussian cost maps representing their personal space. These social costs are used in bid evaluations, and robots actively request team collaboration when humans block their paths.

\section{PRELIMINARIES}\label{sec:prelim}
\subsection{Notation and Definitions}
We denote by $\mathbb{N}$, $\mathbb{R}$, $\mathbb{R}_{\geq 0}$, and $\mathbb{R}_{>0}$ the set of natural, real, real non-negative, and real positive numbers, respectively. Given ${x} \in \mathbb{R}^n$, let $\|{x}\|$ be the Euclidean norm. In the rest of the paper, we will consider a generic polygon $Q \subset \mathbb{R}^2$ as the environment, and ${q} \in Q$ will indicate the position of an arbitrary point in $Q$. We consider a team of $N_r \in \mathbb{N}$ robots, and we denote the set of their positions as $\mathcal{P} = \{{p}_1, \dots, {p}_{N_r}  \} \subset Q$. Taking $R \in \mathbb{R}_{>0}$ as the sensing range of the robots, we denote the set of neighbors for ${p}_i$ as $\mathcal{N}_i = \{ {p}_j \in \mathcal{P} \mid \| {p}_j - {p}_i \| \leq R \}$, for $i=1, \dots, N_r$ and $i \neq j$. We also consider $N_h \in \mathbb{N}$ humans operating in the same environment, whose positions are denoted as ${\human}_i \in Q$ for $i=1, \cdots, N_h$, and $N_o \in \mathbb{N}$ obstacles in positions ${\varrho}_j \in Q$ for $j = 1, \dots, N_o$.
Finally, we use $B({p}_i, r) = \{{q} \in Q \mid \| {q}-{p}_i \| \leq r  \}$ to indicate the closed ball centered in ${p}_i$ with radius $r \in \mathbb{R}_{>0}$.

\subsection{Background}
We will now briefly describe the solution presented in~\cite{pratissoli2022coverage} for coverage control with a limited-range robotic team obeying single-integrator dynamics, $\dot{{p}} = {u}$.\\
This control solution aims to maximize a performance function specifically designed to encode the coverage quality of robots over $Q$. Extending the solution in~\cite{cortes2004coverage}, it is possible to find a limited-range Voronoi partitioning of the environment $\mathcal{V}^r(\mathcal{P})=\{V_1^r, \dots, V_N^r \}$ for robots with sensing range $R$ with:
\begin{equation}\label{eq:voronoi_limited}
    V_i^r({p}_i) = \{{q} \in B({p}_i, r) \mid \| {q} - {p}_i \| \leq \| {q} - {p}_j \|,  \forall {p}_j \in \mathcal{N}_i  \}
\end{equation}
where $r = R/2$. The optimization function $\mathcal{H}^r: Q \rightarrow \mathbb{R}$ is defined as:
\begin{equation}\label{eq:coverage_func}
    \mathcal{H}^r(\mathcal{P})= -\sum_{i=1}^{N_r} \int_{V_i^r} f^r(\|{q}-{p}_i \|) \phi({q})d{q}
\end{equation}
where $\phi: Q \rightarrow \mathbb{R}_{\geq 0}$ is a likelihood function highlighting areas of the environment with greater relevance, and 
\begin{equation}
    f^r(x) = \text{min}(x^2, r^2).
\end{equation}
Following the discussion in~\cite{pratissoli2022coverage}, it is possible to define the control input $u_i$ for each robot that maximizes~\eqref{eq:coverage_func} as:
\begin{equation}\label{eq:input_std}
    {u}_i = k({C}_{V_i^r} - {p}_i),
\end{equation}
where $k \in \mathbb{R}_{>0}$ is a proportional gain, and ${C}_{V_i^r} \in \mathbb{R}^2$ is the centroid of $V_i^r$. Applying Lloyd's algorithm~\cite{lloyd1982least}, the partitioning $\mathcal{V}^r(\mathcal{P})$ is continuously updated until the multi-robot system reaches the so-called \textit{centroidal} Voronoi partitioning, corresponding to all the robots placed on the centroid of their Voronoi cell.  


\subsection{Problem Statement}
We assume robots obey the following motion law: 
\begin{equation}\label{eq:dynamics}
    \dot{x}_i = f(x_i)+g(x_i)u_i
\end{equation}
where $x_i \in \mathbb{R}^n$ is the state of the $i$-th robot, which may include position, velocity, orientation, or a combination of them, $u_i \in U \subset \mathbb{R}^m$ is its control input, $U$ indicates the set of admissible inputs, and $f : \mathbb{R}^n \rightarrow \mathbb{R}^n$ and $g: \mathbb{R}^m \rightarrow \mathbb{R}^{n\times m}$ are Lipschitz continuous functions. 

Robots operate in a 2D environment populated with $N_o \in \mathbb{N}$ obstacles and $N_h \in \mathbb{N}$ humans. The likelihood of events of interest occurring in the environment is modeled by the function $\phi$ as a Gaussian Mixture Model. We assume that a prediction model is available to estimate the future trajectory of each human over a discrete time horizon $T \in \mathbb{N}$, along with the associated uncertainty. Specifically, at the current time $t = t_0$, the predicted position of the $i$-th human at future time step $t_0 + k$ is modeled as a Gaussian distribution:
\begin{equation}\label{eq:human_pred}
    \human_i(t_0 + k \mid t_0) = \mathcal{N}(\mu_i^k, \Sigma_i^k),
\end{equation}
where $k=0, \dots, T-1$ denotes the time index along the prediction horizon. The mean $\mu_i^k \in \mathbb{R}^2$ and the covariance matrix $\Sigma_i^k \in \mathbb{R}^{2\times 2}$ represent the predicted position and the associated uncertainty of the $i$-th human $k$ steps into the future, respectively.

\section{MPC FOR COVERAGE CONTROL}\label{sec:coverage_mpc}
In this section, we present and discuss the proposed optimization-based framework for decentralized coverage control. Specifically, our goal is to use MPC to minimize a specifically designed cost function while satisfying safety and dynamics constraints.
\subsection{Cost Function}
We define the cost function in our optimization problem as the sum of two contributions, $\mathcal{J}_i = \mathcal{J}_i^{\mathrm{cov}} + \mathcal{J}_i^{\mathrm{u}}$. The first replicates the idea of coverage control, and the second encourages smooth control inputs.
\subsubsection{Coverage Cost}
First, we define a suitable coverage cost function that can be minimized in a decentralized manner by each agent. Similarly to the definition of $\mathcal{H}^r$ in~\eqref{eq:coverage_func}, we define the cost function $\mathcal{J}^{\mathrm{cov}}_i: Q \rightarrow \mathbb{R}$ for the $i$-th robot as:
\begin{equation}\label{eq:coverage_cost}
    \mathcal{J}^{\mathrm{cov}}_i = \sum_{k=0}^{T-1} \mathcal{J}^{\mathrm{cov}}_i(k) = \sum_{k=0}^{T-1}\int_{V_i^r(t_0)} \|q-p_i^k \|^2 \phi(q) dq
\end{equation}
where $p_i^k := p_i(t_0+k)$.
It is important to note that the integral in the cost function is consistently evaluated over the limited Voronoi cell computed at the initial time step, $t = t_0$, denoted as $V_i^r(t_0)$, throughout the entire prediction horizon. This approximation is necessary to maintain a fixed integration domain. Without this approximation, the integration domain would change along the horizon, resulting in a nonlinear optimization problem that would be more challenging to solve in real time. It is easy to show that minimizing $\mathcal{J}_i^{\mathrm{cov}}$ returns a set of inputs $\mathcal{U}_i = [ u_i^0, \dots, u_i^{T-1}]$ where $u_i^0$ is defined as in~\eqref{eq:input_std} for single integrator dynamics.
\begin{proposition}
    Consider the $i$-th robot moving in an obstacle-free, convex environment obeying a single integrator dynamics law, $\dot{p}_i = u_i$. The first input $u_i^0$ of the sequence $\mathcal{U}_i$ minimizing~\eqref{eq:coverage_cost} corresponds to~\eqref{eq:input_std}.
\end{proposition}
\begin{proof}
    Following the proof in~\cite[Theorem 1]{pratissoli2022coverage}, we show that~\eqref{eq:input_std} implements a gradient descent of the cost function~\eqref{eq:coverage_cost}. The partial derivatives of the cost function can be computed as follows:
    \begin{align}
        \frac{\partial \mathcal{J}_i^{\mathrm{cov}}}{\partial p_i} &= -2 \int_{V_i^r} (q - p_i) \phi(q)dq \nonumber\\
        ~&= -2M_{V_i^r} (C_{V_i^r} - p_i)
    \end{align}
    where $M_{V_i^r} = \int_{V_i^r} \phi(q)dq$ is the mass of the limited Voronoi cell $V_i^r$.
    Therefore, applying~\eqref{eq:input_std} to the time derivative of the cost function we get:
    \begin{align}
        \dot{\mathcal{J}}_i^{\mathrm{cov}} &= \frac{\partial \mathcal{J}_i^{\mathrm{cov}}}{\partial p_i}\dot{p}_i \nonumber\\ 
        ~&=-[2M_{V_i^r} (C_{V_i^r} - p_i)]^{\top} k(C_{V_i^r} - p_i) \leq 0.
    \end{align}
    Given the quadratic form and the negative sign in front of the expression, the time derivative is non-positive, proving the statement.
\end{proof}
Based on the above considerations, the traditional approach can be regarded as a special case of the proposed method, specifically when the environment is convex and the agent dynamics are modeled as single integrators.

\subsubsection{Control Input Cost}
Given a positive definite input cost matrix $R \in \mathbb{R}^{m\times m}$, we define the second term of the cost function, namely $\mathcal{J}^{\mathrm{u}}_i : \mathbb{R}^m \rightarrow \mathbb{R}$, as:
\begin{equation}\label{eq:ctrl_cost}
\mathcal{J}^{\mathrm{u}}_i = u_i^\top Ru_i .
\end{equation}
This additional term penalizes aggressive maneuvers and energy consumption, making smooth trajectories preferable.

The idea behind the design of $\mathcal{J}_i$ is to maintain the convergence properties of $\mathcal{H}^r$, extensively discussed in the literature, while extending its applicability to more generic motion models, instead of the simplistic single integrator usually adopted in traditional approaches~\cite{cortes2004coverage, pratissoli2022coverage}. 

\subsection{Constraints}
The main advantage of using MPC over a traditional closed-form solution lies in the possibility to encode constraints into the optimization problem. 
\subsubsection{Motion Dynamics}
We want the solution of the optimization problem to follow the dynamic equation~\eqref{eq:dynamics}, which can be written in a discretized form as:
\begin{equation}
    x_i^{k+1} = f\big(x_i^k\big) + g(x_i^k)u_i^k \Delta t, \quad
    \forall k \in \{0, \dots, T-1 \} 
\end{equation}

\subsubsection{Obstacle avoidance}
Obstacle avoidance is not directly addressed in traditional approaches. Usually, a downstream solution is adopted that modifies the control input from coverage control to enhance safety. Instead, we encode safety constraints into the optimization problem to obtain an optimal solution directly applicable to the robot. Given a safety distance $D_s \in \mathbb{R}_{> 0}$, we define the constraint for each obstacle $\varrho_j$ as:
\begin{align}\label{eq:safety}
    \|p_i^k - \varrho_j \|^2 - D_s^2 \geq 0, \quad& \forall j \in \{1, \dots, N_o\} \nonumber \\
    ~& \forall k \in \{0, \dots, T-1 \}
\end{align}

\subsection{MPC Optimization Problem}
Combining all components, we formulate a quadratic optimization problem that yields a sequence of control inputs $\mathcal{U}_i = [u_i^0, \dots, u_i^{T-1}]$ for the $i$-th robot. In accordance with the standard MPC framework, only the first input of the optimized sequence $u_0$ is applied at each time step, and the optimization is repeated at the next iteration. 
\begin{subequations} \label{eq:mpc}
\begin{align}
    \min_{\mathcal{U}_i} ~& \mathcal{J}_i^{\mathrm{cov}} + \mathcal{J}_i^{\mathrm{u}}\label{eq:mpc_cost} \\
    \text{s.t.}~& x_i^{k+1} = f\big(x_i^k\big) + g(x_i^k)u_i^k \Delta t~\label{eq:dynamic_constr} \\
    ~& \|p_i^k - \varrho_j \|^2 - D_s^2 \geq 0 \label{eq:obs_constr} \\
    ~& x_i^k \in \mathcal{X} \label{eq:space_constr} \\
    ~& u_{\mathrm{min}} \preceq u_i^k \preceq u_{\mathrm{max}} \label{eq:input_constr1}
\end{align}
\end{subequations}
In the Quadratic Programming (QP) problem above, the state $x_i$ is constrained to lie within the admissible set $\mathcal{X}$~\eqref{eq:space_constr}, and the control input is subject to the bounds defined in~\eqref{eq:input_constr1}, where the operator $\preceq$ denotes element-wise inequality.

\section{HUMAN-AWARE COVERAGE CONTROL}\label{sec:hmpcc}
Introducing humans into the coordination problem significantly increases complexity. Unlike robots, humans do not share their motion plans, and their behavior is often unpredictable and partially observable. This uncertainty makes it difficult to guarantee safety and maintain performance. To address this, the control framework must incorporate models that predict or bound human motion and adjust robot behavior accordingly. This section describes how we extend the baseline controller to account for the presence of humans in the environment. \\
Specifically, our goal is to coordinate robots to perform coverage control and avoid humans during their motion. Assuming access to a probabilistic predictor for human trajectories (e.g.,~\cite{maeda2023fast, gupta2018social, bae2024singulartrajectory}), each robot can construct a probabilistic map over the future positions of nearby humans, represented as Gaussian distributions characterized by a mean $\mu_i^k$ and covariance matrix $\Sigma_i^k$ for the $k$-th prediction step. The predictor takes as input the past $M \in \mathbb{N}$ steps of the human trajectory and outputs a probability density function over future positions for $T$ steps ahead, where $T$ matches the prediction horizon of the MPC. Given the probabilistic nature of the problem, we design a probabilistic collision avoidance mechanism with chance constraints. Briefly, our goal is to design a constraint of the form
\begin{equation}\label{eq:chance_constr}
    P \big( \| \human_j - p_i \| < D_s \big) \leq \alpha,
\end{equation}
meaning that the probability of the distance between the
$j$-th human and the $i$-th robot falling below the safety threshold $D_s$ must remain below the risk level $\alpha \in [0, 1)$. Starting from the notion of squared Mahalanobis distance $d_M$, defined as:
\begin{equation}\label{eq:mahal_dist}
    d_M^2 = (p_i - \mu_j)^{\top} \Sigma_j^{-1}(p_i - \mu_j)
\end{equation}
evaluating the distance from a point to a probabilistic distribution characterized by mean $\mu_j$ and covariance matrix $\Sigma_j$, we can define our chance constraint for human avoidance as follows (see~\cite{du2011probabilistic}):
\begin{equation}\label{eq:human_constr}
    d_M^2 \geq -2 \ln \Big( \sqrt{\det(2\pi\Sigma_j)} \frac{\alpha}{\pi D_s^2} \Big).
\end{equation}
Furthermore, to account for the growing uncertainty in human motion over time, we introduce a slack variable $\delta^k \in \mathbb{R}_{\geq 0}$  that relaxes the constraint at each prediction step $k$. This allows the constraint to become progressively less conservative further along the horizon. The modified constraint at step $k$ is given by
\begin{equation}\label{eq:human_constr_slack}
    (d_M^k)^2 \geq -2 \ln \Big( \sqrt{\det(2\pi\Sigma_j^k)} \frac{\alpha}{\pi D_s^2} \Big) - \delta^k.    
\end{equation}
Progressive relaxation is achieved by defining an additional cost term $\mathcal{J}_\delta : \mathbb{R} \rightarrow \mathbb{R}$ as:
\begin{equation}
    \mathcal{J}_\delta = \sum_{k=0}^{T-1} w^k \delta^k
\end{equation}
where $w^k \in \mathbb{R}_{>0}$ is a linearly decaying weighting factor defined as $w^k = \Omega \big(1- \frac{k}{T}\big)$, where $\Omega \in \mathbb{R}$ is the initial value.

Finally, the overall optimization problem for human-aware coverage control can be written, adapting~\eqref{eq:mpc}, as:
\begin{subequations} \label{eq:mpc_hr}
\begin{align}
    \min_{\mathcal{U}_i} ~& \mathcal{J}_i^{\mathrm{cov}} + \mathcal{J}_i^{\mathrm{u}} + \mathcal{J}_i^\delta \label{eq:mpc_cost_hr} \\
    \text{s.t.}~& x_i^{k+1} = f\big(x_i^k\big) + g(x_i^k)u_i^k \Delta t~\label{eq:dynamic_constr_hr} \\
    ~& \|p_i^k - \varrho_j \|^2 - D_s^2 \geq 0 \label{eq:obs_constr_hr} \\
    ~& x_i^k \in \mathcal{X} \label{eq:space_constr_hr} \\
    ~& u_{\mathrm{min}} \preceq u_i^k \preceq u_{\mathrm{max}} \label{eq:input_constr} \\
    ~& (d_M^k)^2 \geq -2 \ln \Big( \sqrt{\det(2\pi\Sigma_j^k)} \frac{\alpha}{\pi D_s^2} \Big) - \delta^k. \label{eq:human_cost_hr}
\end{align}
\end{subequations}

\section{EXPERIMENTAL EVALUATION}\label{sec:experiments}
This section evaluates our solution's performance against standard approaches, highlighting the benefits of incorporating trajectory predictions. We will run simulations in three different settings, specifically: \emph{(A)} a convex environment without humans, \emph{(B)} a non-convex environment without humans, and \emph{(C)} a non-convex environment with humans. 
Robots operate in a fully decentralized manner without inter-agent communication. The likelihood function $\phi$ is defined as a Gaussian Mixture Model with random components. The prediction horizon is set to $T = 10$. We employ different dynamic models across experiments to evaluate the adaptability of our solution. Results are evaluated over the following performance metrics:
\begin{itemize}
    \item \textit{Coverage Optimization Function} $\mathcal{H}:Q\rightarrow \mathbb{R}$, corresponding to the range-unlimited version of~\eqref{eq:coverage_func}, indicating how close the team is to the \textit{centroidal Voronoi tessellation}:
    \begin{equation}\label{eq:coverage_func_unlim}
    \mathcal{H}(\mathcal{P})= -\sum_{i=1}^{N_r} \int_{V_i} \|q-p_i \|^2 \phi(q)dq.
    \end{equation}
    \item \textit{Coverage Effectiveness} $\mathcal{E} : Q \rightarrow \mathbb{R}$, assessing how well the likelihood density $\phi(q)$ is covered by the limited-range capabilities of the robots:
    \begin{equation}\label{eq:cov_effect}
        \mathcal{E} = \frac{\sum_{i=1}^N \int_{V_i}\phi(q)dq}{\int_Q \phi(q)dq}.
    \end{equation}
\end{itemize}
We compare our method against a hybrid approach that combines limited-range coverage control~\cite{pratissoli2022coverage} with potential-based repulsion for collision avoidance~\cite{franco2015persistent}. 

\subsection{Convex Environment without Humans}
\begin{figure}[t]
    \centering
    \subfigure[Coverage Function]{\includegraphics[width=0.45\linewidth]{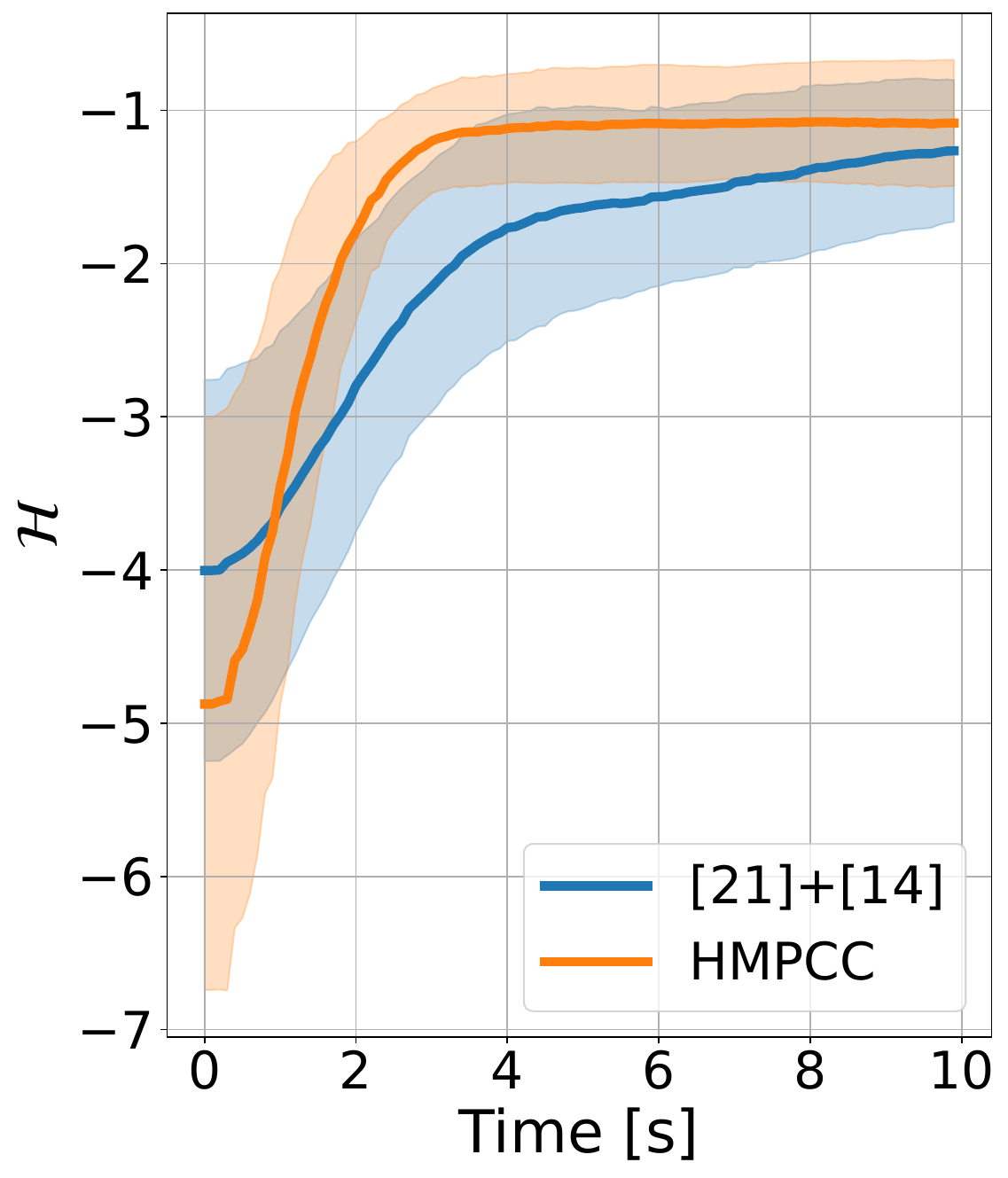}}\quad
    \subfigure[Coverage Effectiveness]{\includegraphics[width=0.45\linewidth]{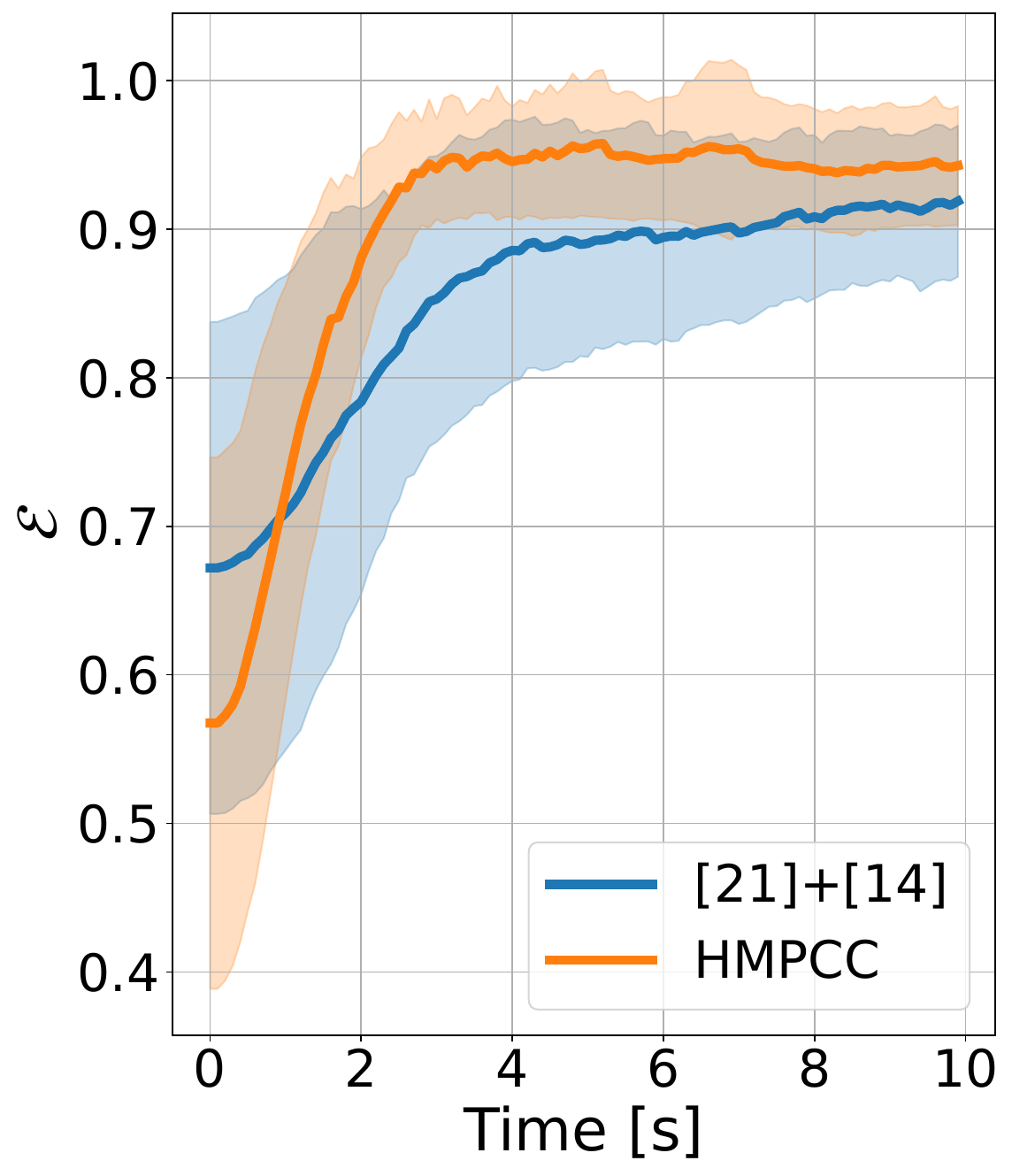}}
    \caption{Convergence rate: comparison between limited-range coverage control~\cite{pratissoli2022coverage} with repulsive collision avoidance~\cite{franco2015persistent} and our HMPCC. Our approach achieves faster convergence for both metrics,  
    }
    \label{fig:conv_rate}
\end{figure}
The first set of simulations was conducted in simplified scenarios commonly assumed by traditional methods, featuring an obstacle-free environment. The experiments are conducted with $N = 6$ robots, randomly initialized in a $10\times 10~\si{m^2}$ environment, obeying a double integrator dynamic law:
\begin{equation}
    {x}_i(t+1) = A{x_i}(t) + Bu_i(t)\Delta t
\end{equation}
where the state of the $i$-th robot ${x_i} = [{p_i}; \dot{{p}_i}]$ contains its position and velocity, and $A = [{0}, {I}; {0}, {0}] \in \mathbb{R}^{4\times 4}$, $B = [{0}; {I}] \in \mathbb{R}^{4\times 2}$. 
Here ${0} \in \mathbb{R}^{2\times 2}$, ${I} \in \mathbb{R}^{2\times 2}$ are the zero and identity matrix, respectively.
The values of the coverage objective function $\mathcal{H}$ and the effectiveness $\mathcal{E}$ were recorded at each time-step over $10$ simulation runs. The duration of each run is set to $10~\si{s}$ with time-step $\Delta t=0.1~\si{s}$. Results are presented in Fig.~\ref{fig:conv_rate} as mean values with standard deviations. As shown in the plots, HMPCC achieves faster convergence for both $\mathcal{H}$ and $\mathcal{E}$. This improved efficiency stems from the fact that the control input is not directly tied to the distance from the centroid. Instead, planning to minimize the cost function~\eqref{eq:coverage_cost} enables robots to move more efficiently toward regions of higher density, resulting in accelerated convergence.

\subsection{Non-Convex Environment without Humans}
\begin{figure}[t]
    \centering
    \subfigure[Baseline]{\includegraphics[width=0.45\linewidth]{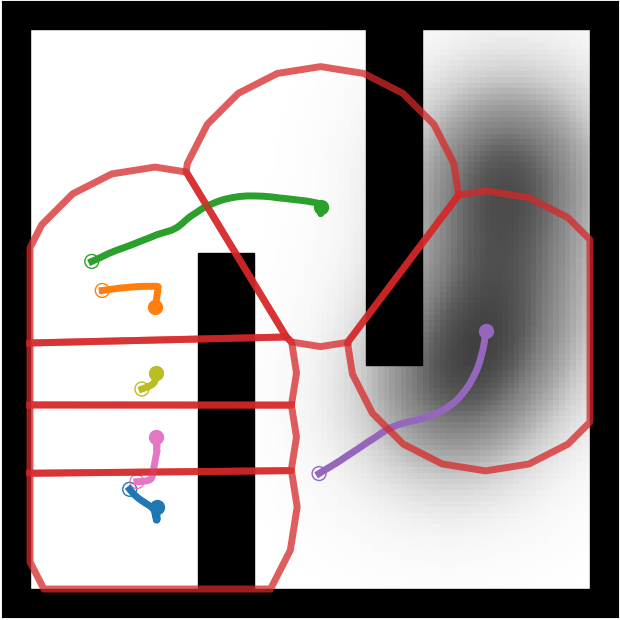}}\quad
    \subfigure[HMPCC]{\includegraphics[width=0.45\linewidth]{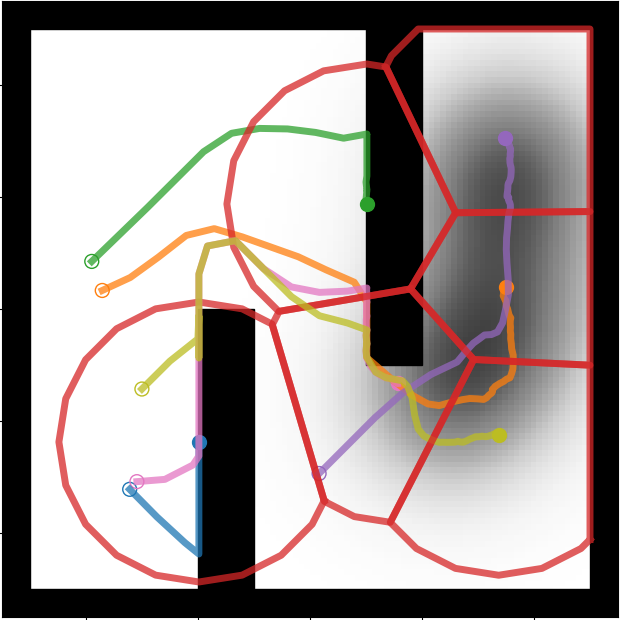}}
    \caption{Comparison in a significant simulation run between limited-range coverage control~\cite{pratissoli2022coverage} with repulsive collision avoidance~\cite{franco2015persistent} and our MPC-based approach. Robots trajectories are depicted with different colors, their starting and final positions are indicated by empty and filled dots, respectively. (a) The baseline is prone to local minima, while (b) HMPCC allows robots to plan obstacle-avoiding trajectories to reach areas of interest.}
    \label{fig:maze_trajectories}
\end{figure}
\begin{figure}[t]
    \centering
    \subfigure[Coverage Function]{\includegraphics[width=0.45\linewidth]{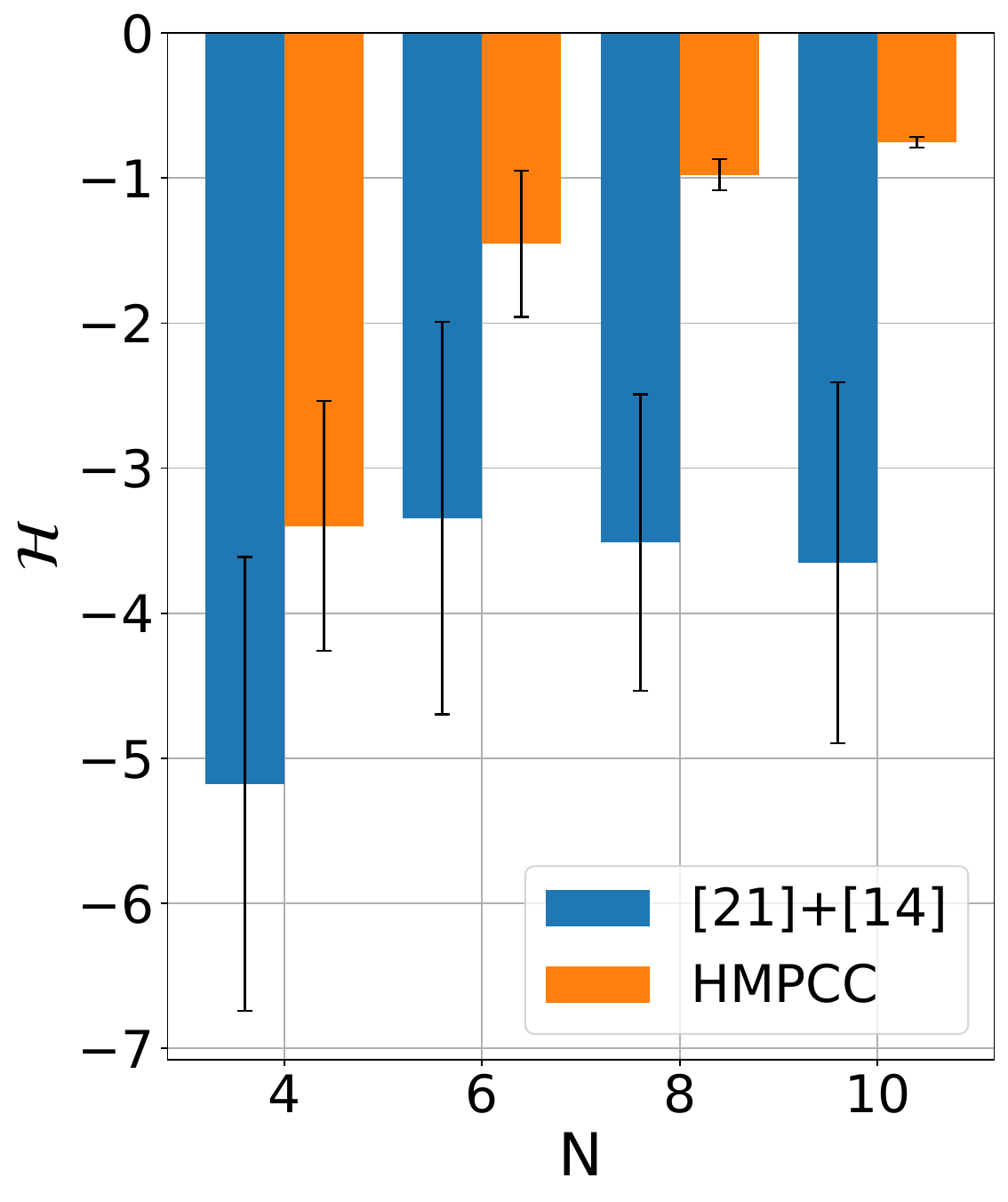}}\quad
    \subfigure[Coverage Effectiveness]{\includegraphics[width=0.45\linewidth]{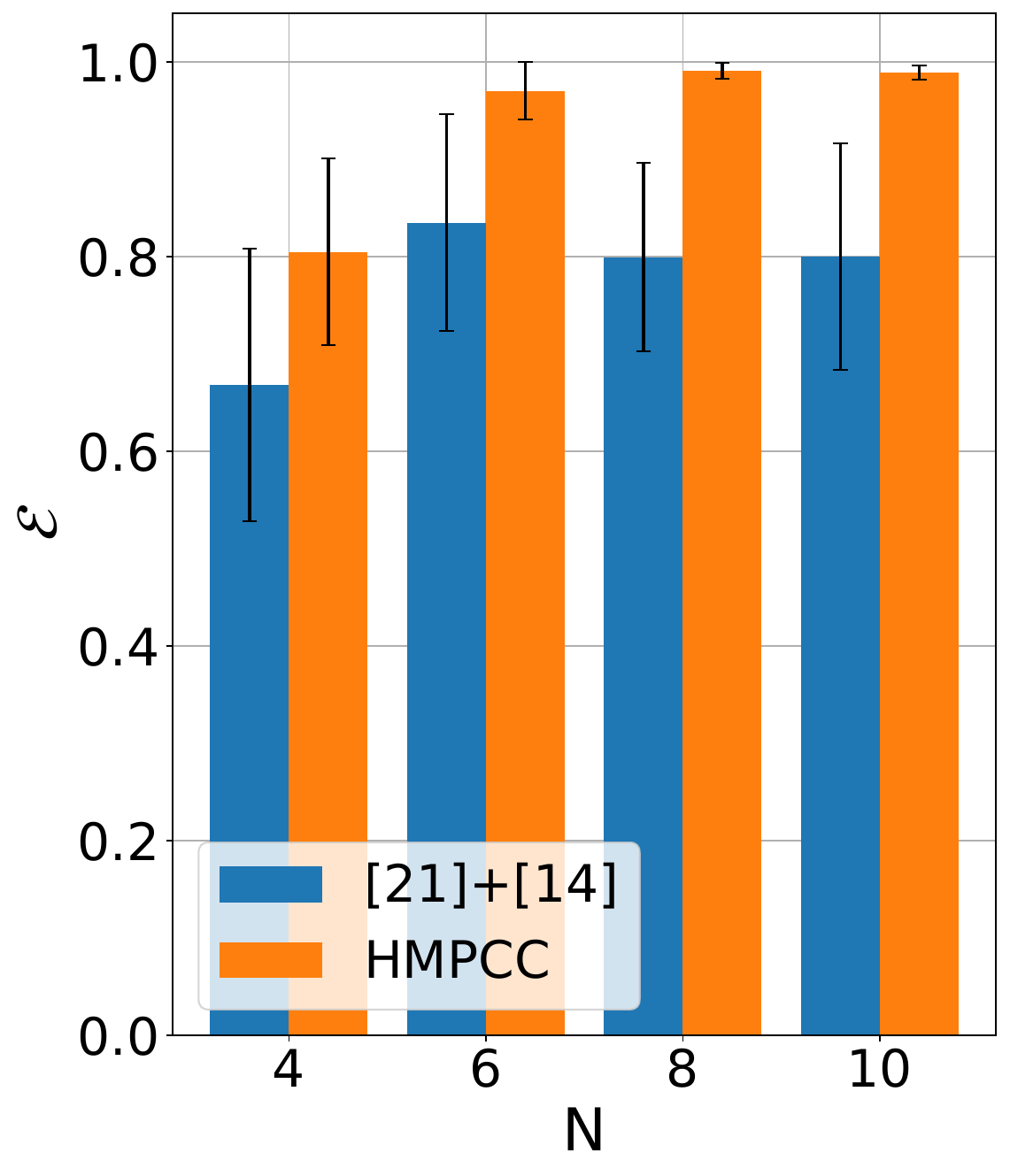}}
    \caption{Evaluation in non-convex environments with increasing number of robots. Our solution outperforms the baseline both in terms of $\mathcal{H}$ (a) and $\mathcal{E}$ (b).}
    \label{fig:nonconv_eff}
\end{figure}
\begin{figure*}[t]
    \centering
    \subfigure[$t=0.0~\si{s}$]{\includegraphics[width=0.18\linewidth]{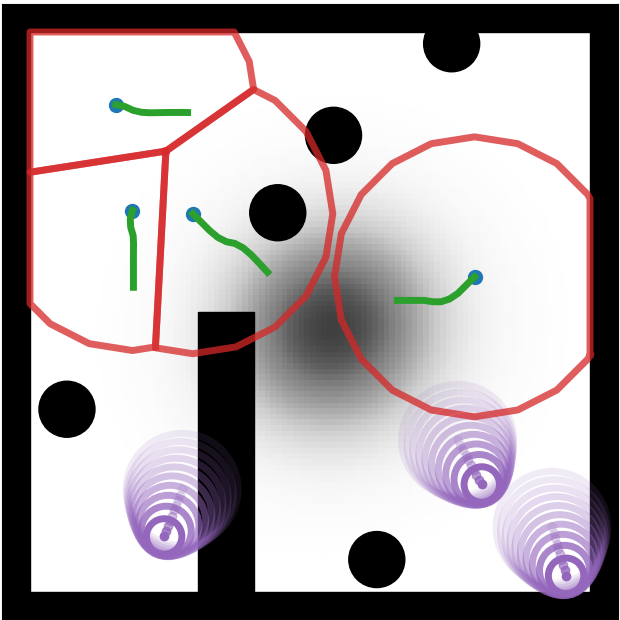}}
    \subfigure[$t=3.0~\si{s}$]{\includegraphics[width=0.18\linewidth]{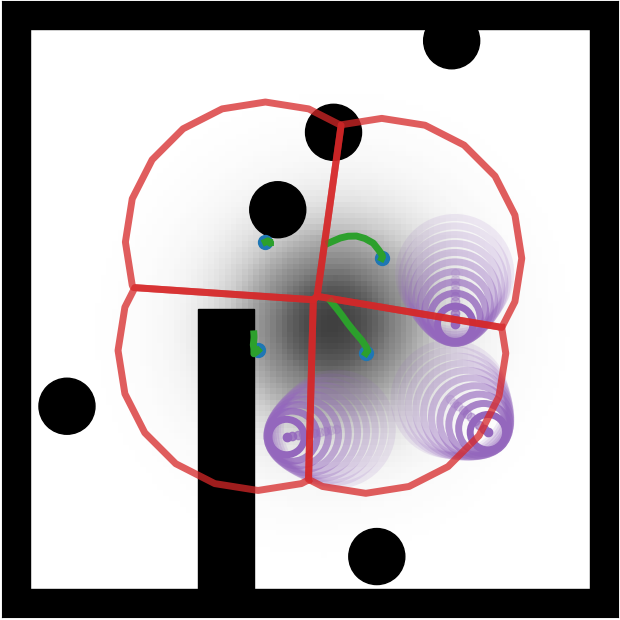}}
    \subfigure[$t=6.5~\si{s}$]{\includegraphics[width=0.18\linewidth]{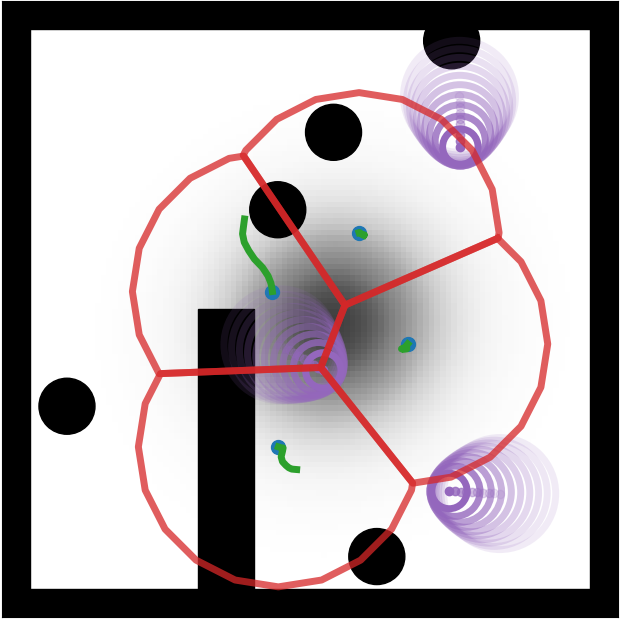}}
    \subfigure[$t=9.0~\si{s}$]{\includegraphics[width=0.18\linewidth]{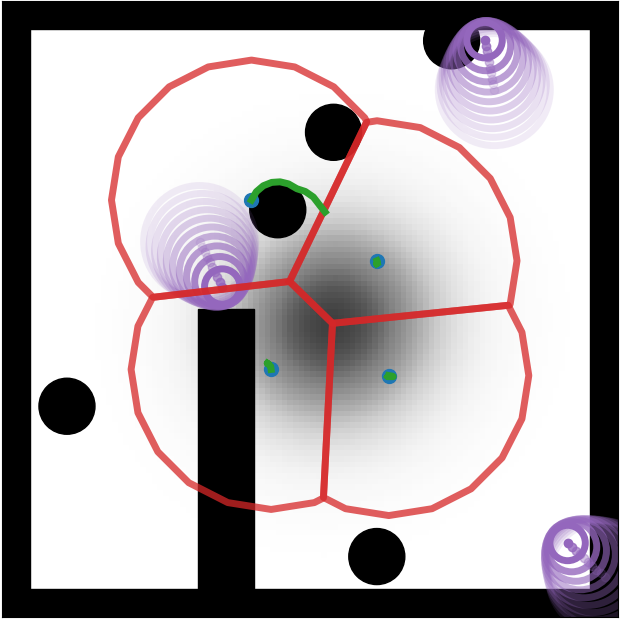}}
    \subfigure[$t=11.0~\si{s}$]{\includegraphics[width=0.18\linewidth]{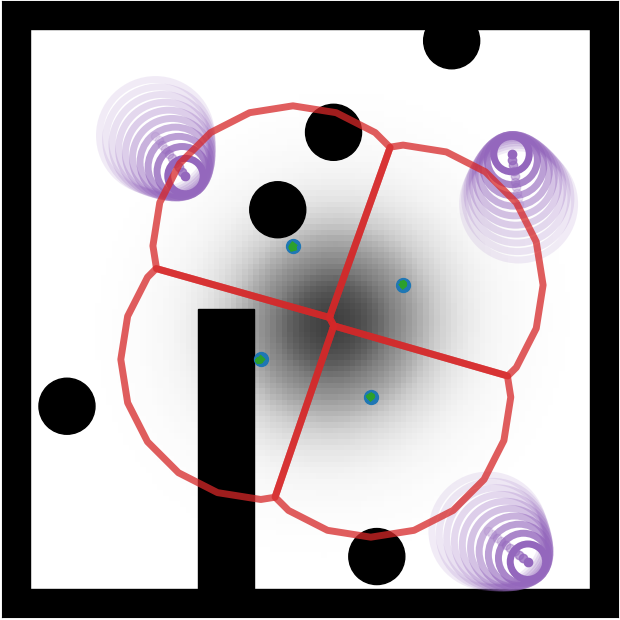}}
    \caption{Snapshots of a representative task execution involving a team of $4$ unicycle robots and $3$ humans. Robots are shown as blue dots, their planned trajectories as green curves, and humans in purple. The uncertainty associated with human motion is visualized as expanding circles along the prediction horizon. (a) Robots start in random positions, (b) reach the area of interest, (c) temporary disperse to make room for approaching humans, (d) avoid obstacles in the environment, and (e) finally return to the area of interest.}
    \label{fig:hr_snaps}
\end{figure*}
As previously noted, a major limitation of traditional coverage control approaches is their reliance on the assumption of convex environments. When obstacles are present, the addition of a separate obstacle-avoidance mechanism often causes robots to become trapped in local minima, preventing them from reaching the desired coverage areas. To test how HMPCC behaves in those scenarios, we ran another set of simulations, increasing the duration of each run to $15~\si{s}$. We employ teams with $N = \{4, 6, 8, 10 \}$ robots following single integrator dynamics $p_i(t+1)=p_i(t)+u_i(t)\Delta t$.\\  Fig.~\ref{fig:maze_trajectories} illustrates the difference between the baseline hybrid approach, combining~\cite{pratissoli2022coverage} and~\cite{franco2015persistent}, shown in Fig.~\ref{fig:maze_trajectories}a, and our proposed method in Fig.~\ref{fig:maze_trajectories}b. It is evident that the baseline controller gets stuck when direct paths are obstructed, whereas our constrained optimization-based approach enables the robots to navigate around obstacles and reach the target areas. As a result, HMPCC achieves superior final performance, as illustrated in Fig.~\ref{fig:nonconv_eff}a for the coverage function $\mathcal{H}$, and Fig.~\ref{fig:nonconv_eff}b for the effectiveness $\mathcal{E}$.

\subsection{Non-Convex Environment with Humans}
Finally, we introduced humans into the environment to evaluate how the proposed method benefits from predicting their trajectories. Similarly to the previous approach, we ran sets of $10$ simulations each to compare HMPCC with the baseline without prediction. We set the number of robots to $N = 6$ and humans to $N_h = \{3, 6, 9 \}$, and the risk factor for the chance constraint~\eqref{eq:chance_constr} to $\alpha = 0.1$. Unicycle robots were employed in this evaluation phase, with a motion law as follows:
\begin{align}
    x_i(t+1) = x_i(t) + \begin{bmatrix}
        \cos\theta_i & 0 \\
        \sin\theta_i & 0 \\
        0 & 1
    \end{bmatrix} \begin{bmatrix}
        v_i \\ \omega_i
    \end{bmatrix} \Delta t,
\end{align}
where the state $x_i = [p_i; \theta_i] \in R^3$ indicates position and orientation of the $i$-th robot, while $v_i, \omega_i \in \mathbb{R}$ are its linear and angular velocity, respectively. Instead, we model human motion as:
\begin{align}
    \human_j(t+1) &= \human_j(t) + \begin{bmatrix}
        \cos \vartheta_j \\
        \sin \vartheta_j
    \end{bmatrix} \nu_j \Delta t \\
    \vartheta_j(t+1) &= \vartheta_j(t)+\psi_j \Delta t
\end{align}
with $\nu_j \in \mathbb{R}_{\geq 0}$ and $\psi_j \sim \mathcal{N}(0, \sigma)$ being a constant linear velocity and a random angular velocity, respectively. Human motion prediction was modeled using a constant-velocity assumption, resulting in the probabilistic trajectory defined in~\eqref{eq:human_pred}, with the following mean and covariance update:
\begin{align}
    \mu_j^k = \mu_j^{k-1} + \begin{bmatrix}
        \cos \vartheta_j^{k-1} \\
        \sin \vartheta_j^{k-1}
    \end{bmatrix} \hat{\nu}_j, &&
\Sigma_j^k = \Sigma_j^{k-1}+Q,
\end{align}
\renewcommand{\arraystretch}{1.2}
\begin{table}[]
\centering
\caption{Results in human-filled scenarios.}
\begin{tabular}{|c|c|c|cc|}
\hline
$N_h$ & $\mathcal{E}$     & $\mathcal{H}$      & \multicolumn{2}{c|}{Success rate}          \\ \hline
  & \textbf{HMPCC} & \textbf{HMPCC}  & \multicolumn{1}{c|}{\textbf{Baseline}} & \textbf{HMPCC} \\ \hline
3 & 0.952 & -1.110 & \multicolumn{1}{c|}{$60\%$}       & $100\%$   \\ \hline
6 & 0.942 & -1.613 & \multicolumn{1}{c|}{$20\%$}       & $100\%$   \\ \hline
9 & 0.875 & -1.993 & \multicolumn{1}{c|}{$40\%$}       & $80\%$    \\ \hline
\end{tabular}
\label{tab:results}
\end{table}
where $Q \in \mathbb{R}^{2 \times 2}$ is a noise matrix that accounts for growing uncertainty over the prediction horizon, and $\hat{\nu}_j \in \mathbb{R}_{\geq 0}$ is the $j$-th human velocity predicted from the last $M$ steps:
\begin{equation}\label{eq:vel_estimate}
    \hat{\nu}_j^t = \frac{\|\human_j^t - \human_j^{t-M+1}\|}{(M-1) \Delta t}
\end{equation}
In this setting, we also analyzed the success rate across multiple simulation runs to compare our method with the baseline. Execution was considered unsuccessful if at least one robot collided with an obstacle or a human, or exited the boundaries of the environment. A summary of the results is provided in Table~\ref{tab:results}. It is important to note that the values of $\mathcal{E}$ and $\mathcal{H}$ are captured at the end of each simulation, which is fixed to $15~\si{s}$. For this reason, robots may have temporarily moved away from high-density areas to avoid collisions with humans. Moreover, we did not report the metrics for the baseline method due to an insufficient number of successful executions, rendering the results statistically insignificant. Additionally, Fig.~\ref{fig:hr_snaps} presents snapshots of a representative task execution, illustrating how the robots plan trajectories that avoid both obstacles and humans, and temporarily relocate from areas of interest to make room for approaching humans.\\
Finally, we tested HMPCC in a physical simulation using Gazebo and TurtleBot3 unicycle robots to assess the computational feasibility of our approach. We prepared a simulation with $3$ robots and $3$ humans in a $20 \times 20~\si{m^2}$ environment with obstacles, as illustrated in Fig.~\ref{fig:gazebo}. The simulation was run on a PC equipped with an \textit{Intel Core i7} CPU, an \textit{NVIDIA GeForce RTX 4060} GPU, and $16~\si{GB}$ of RAM, using ROS for the implementation. We set the control frequency of the robots to $10~\si{Hz}$. A representative run of the simulation is reported in the accompanying video.

\subsection*{Summary of the Results}
In summary, the presented results demonstrate several advantages of our method over traditional coverage control approaches:
\begin{itemize}
    \item \textbf{faster convergence} -- Even in simple environments, an optimization-based approach appears to be more efficient in reaching target areas;
    \item \textbf{robustness to complexity} -- By encoding dynamics and obstacle avoidance as constraints, our method handles non-convex environments and varied motion models, reducing susceptibility to local minima; and
    \item \textbf{human-aware deployment} -- The ability to predict the area where humans are more likely to be in the future fits with the planning strategy to efficiently avoid them.
\end{itemize}

\section{CONCLUSIONS}\label{sec:conclusion}
In this paper, we introduced HMPCC, a human-aware optimization-based coverage control strategy. By formulating collision avoidance and robot dynamics as constraints, HMPCC overcomes key limitations of traditional methods. The integration of human trajectory prediction ensures safe operation even when human intentions are unknown. Experimental results show that HMPCC outperforms a traditional baseline in both safety and coverage efficiency. Future work will focus on integrating more advanced prediction models (e.g., FlowChain~\cite{maeda2023fast}), exploring human-robot collaborative scenarios, and validating the approach through real-world experiments to assess the computational feasibility of our solution.







\bibliographystyle{IEEEtran}
\bibliography{biblio}

\end{document}